\documentclass{sig-alternate-2013}
\usepackage{pslatex}
\usepackage{graphicx}
\usepackage{amsmath,amssymb}
\usepackage{multirow}
\usepackage{verbatim}
\usepackage{epstopdf}
\usepackage{epsfig}
\usepackage{cite}
\usepackage{color}
\usepackage{url}
\usepackage{subfigure}
\usepackage{algorithm2e}
\usepackage{algorithmic}
\usepackage{xparse}
\NewDocumentCommand{\ceil}{s O{} m}{%
  \IfBooleanTF{#1} 
    {\left\lceil#3\right\rceil} 
    {#2\lceil#3#2\rceil} 
}

\DeclareMathOperator{\argmax}{argmax}

\newcommand{\st}{\text{s.t.}}

\newtheorem{theorem}{Theorem}
\newtheorem{remark}[theorem]{Remark}
\newtheorem{lemma}[theorem]{Lemma}

\usepackage{ragged2e}
\ifodd 0
\newcommand{\rev}[1]{{\color{blue}#1}} 
\newcommand{\com}[1]{\textbf{\color{red}(COMMENT: #1)}} 
\newcommand{\clar}[1]{\textbf{\color{green}(NEED CLARIFICATION: #1)}}
\newcommand{\response}[1]{\textbf{\color{magenta}(RESPONSE: #1)}} 
\else
\newcommand{\rev}[1]{#1}
\newcommand{\com}[1]{}
\newcommand{\clar}[1]{}
\newcommand{\response}[1]{}
\fi

\newfont{\mycrnotice}{ptmr8t at 7pt}
\newfont{\myconfname}{ptmri8t at 7pt}
\let\crnotice\mycrnotice%
\let\confname\myconfname%

\permission{Permission to make digital or hard copies of all or part of this work for personal or classroom use is granted without fee provided that copies are not made or distributed for profit or commercial advantage and that copies bear this notice and the full citation on the first page. Copyrights for components of this work owned by others than ACM must be honored. Abstracting with credit is permitted. To copy otherwise, or republish, to post on servers or to redistribute to lists, requires prior specific permission and/or a fee. Request permissions from permissions@acm.org.}
\conferenceinfo{MobiHoc'15,}{June 22--25, 2015, Hangzhou, China.}
\copyrightetc{Copyright \copyright~2015 ACM \the\acmcopyr}
\crdata{978-1-4503-3489-1/15/06\ ...\$15.00.\\
http://dx.doi.org/10.1145/2746285.2746287}

\begin{document}
\setlength{\pdfpagewidth}{8.5in}
\setlength{\pdfpageheight}{11in}
\title{An Online Approach to Dynamic Channel Access and Transmission Scheduling}

\author{Yang Liu, Mingyan Liu\\
Department of Electrical Engineering and Computer Science \\
University of Michigan, Ann Arbor\\
Email: \{youngliu,mingyan\}@eecs.umich.edu\\}

%

\maketitle

\begin{abstract}
Making judicious channel access and transmission scheduling decisions is essential for improving performance (delay, throughput, etc.) as well as energy and spectral efficiency in multichannel wireless systems.  This problem has been a subject of extensive study in the past decade, and the resulting dynamic and opportunistic channel access schemes can bring potentially significant improvement over traditional schemes.  However, a common and severe limitation of these dynamic schemes is that they almost always require some form of a priori knowledge of the channel statistics. A natural remedy is a learning framework, which has also been extensively studied in the same context, but a typical learning algorithm in this literature seeks only the best static policy (i.e., to stay in the best channel), with performance measured by {\em weak regret}, rather than learning a good dynamic channel access policy.  There is thus a clear disconnect between what an optimal channel access policy can achieve with known channel statistics that actively exploits temporal, spatial and spectral diversity, and what a typical existing learning algorithm aims for, which is the static use of a single channel devoid of diversity gain.
In this paper we bridge this gap by designing learning algorithms that track known optimal or sub-optimal dynamic channel access and transmission scheduling policies, thereby yielding performance measured by a form of {\em strong regret}, the accumulated difference between the reward returned by an optimal solution when a priori information is available and that by our online algorithm. We do so in the context of two specific algorithms that appeared in \cite{Zheng:2009:DOS:1669334.1669354} and \cite{Chang07optimalchannel}, respectively, the former for a multiuser single-channel setting and the latter for a single-user multichannel setting. In both cases we show that our algorithms achieve sub-linear regret uniform in time and outperforms the standard weak-regret learning algorithms. 

\end{abstract}




\section{Introduction}

Making judicious channel access and transmission scheduling decisions is essential for improving performance (delay, throughput, etc.) as well as energy and spectral efficiency in wireless systems, especially those consisting of multiple users and multiple channels.  Such decisions are often non-trivial because of the time-varying nature of the wireless channel condition, which further varies across different users and different spectrum bands.  Such temporal, spatial and spectral diversity provide opportunities for a radio transceiver to exploit for performance gain and the past decade has seen many research advances in this area. For instance, a transmitter can seek the best channel through channel sensing before transmission, see e.g.,  \cite{10.1109/BROADNETS.2004.46, Shu:2009:TSC:1614320.1614325,Sabharwal06opportunisticspectral} for such dynamic multi-channel MAC schemes that allow transmitters to opportunistically switch between channels in search of good instantaneous channel condition; if a transmitter consistently selects a channel with better instantaneous condition (e.g., higher instantaneous received SNR) from a set of channels, then over time it sees (potentially much) higher average rate \cite{Liu03opportunisticfair, 900644,916290}.  Similarly, a transmitter can postpone transmission if the sensed instantaneous condition is poor in hopes of better condition later, see e.g., \cite{Zheng:2009:DOS:1669334.1669354} for stopping rule based sequential channel sensing policies, in the single-user multichannel and single-channel multiuser scenarios, respectively.  Variations on the same theme include \cite{tan2010distributed} where a distributed opportunistic scheduling problem under delay constraints is investigated, and \cite{Chang07optimalchannel} where a generalized stopping rule is developed for the single-user multichannel setting. 

These dynamic channel access schemes (both optimal and sub-optimal) improve upon traditional schemes such as channel splitting \cite{1208698,1285260}, multi-channel CSMA \cite{111441}, and multi-rate systems  \cite{Kamerman1997}.  
However, a common and severe limitation of these dynamic schemes is that they almost always require some form of a priori knowledge of the channel statistics.  For instance, a typical assumption is that the channel conditions evolve as an IID process and that its distribution for each channel is known to the transmitter/user, see e.g., \cite{Sabharwal06opportunisticspectral, Zheng:2009:DOS:1669334.1669354, Chang07optimalchannel}.  While in some limited setting such information may be acquired with accuracy and low latency, this assumption does not generally hold.  Furthermore, the channel statistics may be time-varying, in which case such an assumption can only be justified if there exists a separate channel sampling process which keeps the assumed channel statistics information up to date.   


To relax such an assumption, it is therefore natural to cast the dynamic channel sensing and transmission scheduling problem in a learning context, where the user is not required to possess a priori channel statistics but will try to learn as actions are taken and observations  are made.  Within this context, the type of {\em online} learning or {\em regret} learning, also often referred to as the Multi-Armed Bandit (MAB) \cite{LR85,Anantharam:M86/62,Auer:2002:FAM:599614.599677} framework, is particularly attractive, as it allows a user to optimize its performance throughout its learning process.  
%
For this reason, this learning framework has also been extensively studied within the context of multichannel dynamic spectrum access, see e.g., \cite{tekin2011online} for single-user and \cite{6415693,liu2011learning} for multiuser settings. 
However, in most of this literature, the purpose of the learning algorithm is for a transmitter to find the best channel in terms of its average condition and then use this channel for transmission majority of the time.  It follows that the performance of such learning algorithms is measured by {\em weak regret}, the difference between a learning algorithm and the best single-action policy which in this context is to always use the channel with the best average condition.  Accordingly, the key ingredient  in these algorithms is to form accurate estimates on the average condition for each channel.

We therefore see a clear disconnect between what an optimal channel access policy can achieve with known channel statistics (e.g., by employing a stopping rule based algorithm) that actively exploits temporal, spatial and spectral diversity, and what a typical existing learning algorithm aims for, i.e., essentially the static use of a single channel, which unfortunately completely eliminates the utilization of diversity gain\footnote{Some multiuser learning algorithms attempts to separate users into different channels, so do exploit to some degree the multiuser diversity gain, see e.g., \cite{tekin2012performance}.}.  

Our goal is to bridge this gap and seek to design learning algorithms that instead of trying to track the best average-condition channel, attempt to track a known optimal or sub-optimal channel access and transmission scheduling algorithm, thereby yielding performance measured by a form of {\em strong regret}.  Our presentation and analysis strongly suggest that such learning algorithms may be constructed in a much broader context, i.e., they can be made to track any prescribed policy and not just those cited earlier or even limited to the dynamic spectrum access context.  However, to make our discussion concrete, we shall present our results in the context of specific channel sensing and access algorithms.  

Specifically, we present the general framework of such a learning algorithm, followed by the detailed instances designed to track the stopping rule policies given in \cite{Zheng:2009:DOS:1669334.1669354} and \cite{Chang07optimalchannel}, respectively.  
The choice of these two algorithms is not an arbitrary one.  Our intention is to use two representatives to capture a fairly wide array of similar algorithms of this kind.  The stopping rule algorithm in  \cite{Zheng:2009:DOS:1669334.1669354} is a relatively simple one, designed for multiple users competing for access to a single channel; it exploits temporal and spatial (multiuser) diversity, the idea being for a user to defer transmission if it perceives poor channel quality thereby giving the opportunity to another user with better conditions.  The stopping rule algorithm in \cite{Chang07optimalchannel}, on the other hand, is much more complex in construction; it is designed for a single user with access to multiple channels by exploiting spectral and temporal diversity, the idea being to find the channel with the best instantaneous condition. Both algorithms assume that channel qualities evolve in an IID fashion with known probability distributions, though different channels may have different statistics \cite{Chang07optimalchannel}; and both are provably optimal (or near-optimal) under mild technical conditions. 
For other stopping-rule based policies see also \cite{Sabharwal06opportunisticspectral,10.1109/BROADNETS.2004.46,Shu:2009:TSC:1614320.1614325}. We show that in both cases our algorithms achieve a sub-linear accumulative strong regret (against their respective reference algorithms from \cite{Zheng:2009:DOS:1669334.1669354} and \cite{Chang07optimalchannel}), thus achieving zero-regret averaged over time. \rev{ In this paper, we do not consider interferences from multiple users, that is we consider cases with either a single user or non-strategic and collaborative users. It is however another interesting direction of applying regret learning results to scheduling problems. In such case, adversarial models will be needed to capture the effects of interference when multiple transmitters present in the system. In particular, in \cite{asgeirsson2011game} Asgeirsson et al. studied a capacity maximization problem in distributed wireless network under SINR interference model and show a constant factor approximation bound compared to the global optimum is achievable. In \cite{dams:spaa12} Dams et al. proposed scheduling algorithms for a similar problem but under Reyleigh-fading interference models and show a logarithmic order approximation. Then in a later work \cite{dams2014jamming}, the same authors extend their results to when there exists adversarial jammer. }

The rest of the paper is organized as follows. Problem formulation is presented in Section \ref{sec:pf}, and the two reference optimal offline algorithms in Section \ref{sec:off}. We present our online learning algorithms in Section \ref{sec:online} with performance analysis given in Section \ref{sec:regret}.  Numerical results are given in Sections \ref{sec:simulation} and \rev{we discuss several possible extensions of our work in Section \ref{sec:discuss}}. Section \ref{sec:conclude} concludes the paper.

\section{Problem formulation}\label{sec:pf}

In this section we present two system models and their corresponding transmission scheduling problems. This lays the foundation for us to introduce the two offline optimal stopping-rule policies from \cite{Zheng:2009:DOS:1669334.1669354} and \cite{Chang07optimalchannel}, respectively in Section \ref{sec:off}; these are the policies our learning algorithm presented in Section \ref{sec:online} aims to track.  

\subsection{Model I: multiuser, single-channel}

Under the first model (studied in \cite{Zheng:2009:DOS:1669334.1669354}), there is a finite number of users/transmitters, indexed by the set $\mathcal M = \{1,2,...,M\}$, $M \geq 1$, and a single channel.  The system works in discrete time slots indexed by $n=1,2, \cdots$. Denote the channel quality by $X(n)$, $n=1,2,\cdots$. This quantity measures how good a channel is; for example, $X(n)$ could model the Signal-Noise-Ratio (SNR) for the channel at time $n$. At time $n$, if no one is transmitting on the channel, a user $i \in \mathcal M$ attempts to access with probability $0\leq p_i\leq 1$ by sending a carrier sensing packet. A carrier sensing period takes a constant amount of time denoted by $\zeta$ (slots). The contention resolution is done by random access, i.e. an access attempt is successful with probability
$
p_s = \sum_{i\in \mathcal M} p_i \cdot \prod_{j \neq i} (1-p_i),
$
when there is only one user attempting access. Denote the random contention time between two successful accesses by $\eta$; it follows that $E[\eta] = \zeta/p_s$ (slots). We assume the process $\{X(n_k)\}_{k=1,2,...,}$ forms an IID process, where $n_k$ is the time the $k$-th contention succeeds. That is we assume the samples collected at successful accesses are generated in an IID fashion (as assumed and argued in \cite{Zheng:2009:DOS:1669334.1669354}).  Upon a collision, the current slot will be abandoned and users re-compete in the next time slot. On the other hand, users keep silent if there is an active transmission on the channel. For simplicity it is assumed that $X(n)$ stays unchanged during each transmission, which may be justified if transmission times are kept on a smaller time scale than channel coherence times \cite{zheng2007channel}.
 Once a user gains access right (and sees the channel quality $X(n)$), it has two options: 
\begin{itemize}
\item access the channel right away for $K$ \rev{time slots} (\emph{stop}); or 
\item give up the access opportunity, release the channel for all users to re-compete  (\emph{continue}).  
\end{itemize}
This can be more formally stated as an optimal stopping rule (OSR) problem: users decide at which time to stop the decision process and use the channel. There are a number of variations of this problem with slightly different model, see e.g., \cite{Shu:2009:TSC:1614320.1614325}. The idea is when the channel quality is poor, a user would give up the transmission opportunity so that it is more likely that a user with better perceived channel quality will get to use it.  Denote the stopping time by $\tau$, then the objective is to design a stopping rule for all users so as to maximize the rate-of-return, which is the effective data rate for each successful access (\cite{Zheng:2009:DOS:1669334.1669354})
\begin{align}
J^*_{I} = \max_{\tau \in \Pi} J^{\tau}_{I} &= \max_{\tau \in \Pi} E\biggl[\frac{X(\sum_{k=1}^{\tau}\eta_k)\cdot K}{K_{\tau}}\biggl]\nonumber \\
& \underbrace{\longrightarrow}_{\text{Renewal theory}} \rev{\max_{\tau \in \Pi}} \frac{E[X(\sum_{k=1}^{\tau}\eta_k)\cdot K]}{E[K_{\tau}]}~,\label{obj2}
\end{align}
where $\Pi$ is the strategy space and $\eta_k$ is the $k$-th contention time and $K_{\tau} = \sum_{k=1}^{\tau}\eta_k + K$ is the total amount of time spent for each successful transmission. 

In this model, we regard the decision process between two consecutive successful transmissions (note that no successful transmission occurs if a user who wins access forgoes the transmission opportunity) as one \emph{meta stage}. Suppose there are all together $H$ meta stages (thus $H$ successful transmissions).  We define the following \emph{strong regret} performance measure, 
\begin{align*}
 R_{I}(H)&= \sup_{\mathbf{\tau} \in \Pi^H}  E^{\mathbf{\tau}}\biggl[ \sum_{l=1}^H \frac{X(\sum_{k=1}^{\tau_l}\eta_k)\cdot K}{K_{\tau_l}}\biggr]\nonumber \\
 &-  E^{\mathbf{\alpha}}\biggl[ \sum_{l=1}^H \frac{X(\sum_{k=1}^{\alpha_l}\eta_k)\cdot K}{K_{\alpha_l}} |\mathcal F_{l-1},...,\mathcal F_0\biggr]~.
\end{align*}
In above formulation, since channel conditions are IID over time, for each meta stage we restart the clock, i.e., we always set the first time slot for each meta stage as $n=1$. $\alpha_l$ is the stopping time for the $l$-th meta stage and $X(\sum_{k=1}^{\alpha_l}\eta_k)$ is the corresponding reward. Here we denote by the $\mathcal F_{l} := \cup_{j \in \mathcal M} \mathcal F^j_{l}$ the set of observations of channel qualities at meta stage $l$ ( with $\mathcal F^j_{l}$ for each user $j$).

\subsection{Model II: single-user, multichannel}

Under the second model (studied in  \cite{Chang07optimalchannel}), there is a finite number of channels, denoted and indexed by $\mathcal O = \{1,2,...,N\}$, each of which yields a non-negative reward when selected for transmission (e.g., throughput, delay etc). For any subset $S \subseteq \mathcal O$ we will use $\mathcal{O}-S$ to denote the set $\{j : j \in \mathcal O~\&~j \notin S\}$. There is one decision maker (user/transmitter) within the system. 
 The system again works in discrete time slots $n=1,2,...,\tau \leq N$; these however are much smaller time units than those under Model I because they are used only for channel sensing and not transmission. The user  sequentially chooses a set of channels to probe for their condition, stops at a stopping time $\tau$ using certain stopping rule, and selects a channel for transmission (over a period of time larger than a slot).  The decision process thus consists of determining in which sequence to sense the channels, when to stop, and which channel to use for transmission when stopping. 

For consistency we reuse the terminology \emph{meta stage} to describe the above decision process between $n=1$ and $\tau$; this will be referred to as one \emph{meta stage}.  Each time a new meta stage starts the clock is reset to $n=1$. The meta stages are indexed by $t=1,2,...,T$.  There is a period of transmission between two successive meta stages.  It is assumed that the channel condition remains constant within a single meta stage and forms an IID process over successive meta stages.  This is modeled by a reward $X_i$ (to generate $\{X_i(t)\}_t$) for the $i$-th channel given by a pdf $f_{X_i}(\cdot)$ and cdf $F_{X_i}(\cdot)$, respectively. 
Channels are independent of each other, i.e., a specific channel $i$'s realization $X_{i}(t;\omega)$ does not reveal any information for channels in $\mathcal O - \{i\}$.  



The transmitter is able to sense one channel (to observe $X_i$) at each decision step $n$ with a finite and constant sensing cost $c_i \geq 0$ for each channel $i$. 
The system works in the following way at each $n$ of meta stage $t$: The transmitter makes a decision between the following choices: 
\begin{itemize}
\item continues sensing; if this is the case then furthermore decide which channel to probe next  (\emph{sense});
\item stops sensing and proceeds to transmit (\emph{access}). Under this case there are two more options to choose from: 
\begin{itemize} 
	\item access the channel with the best observed instantaneous condition (\emph{access with recall});
	\item access the best channel (with highest expected reward) from the un-probed set without sensing (\emph{access with guess}).
\end{itemize}
\end{itemize} 

For the offline problem, due to the IID assumption on the channel condition, the decision strategy at each meta stage $t$ is the same. We thus suppress the time index $t$; 
the transmitter's objective is to choose the strategy that maximizes the collected reward minus the sum of probing costs: 
\begin{align}
J^*_{II} = \max_{\pi \in \Pi} J^{\pi}_{II}  = \max_{\pi \in \Pi} E \biggl[X_{\pi(\tau)}-\sum_{n=1}^{\tau-1} c_{\pi(n)} \biggr]~,\label{obj}
\end{align}
 where $\pi$ denotes a probing strategy and $\tau$ the stopping time. 
From \cite{Chang07optimalchannel}, it can be shown for time slots $n=1,2,...,\tau$ at any meta stage $t$, a sufficient information state is given by the pair $(x(n), S_n)$ where $S_n$ is the un-probed channel set and $x(n)$ is the highest observed reward among the set of probed channels $\mathcal O-S_n$. 
Let $V(x,S)$ denote the value function, the maximum expected remaining reward given the system state is $(x,S)$, the problem/decision process at the $n$-th decision step is equivalent to the following dynamic programming (DP) formulation\\[-15pt]
\begin{align}
V(x(n), S_n) = \max \biggl\{ \max_{j \in S_n}&\{-c_j + E[V(\max\{x(n), X_j\},S_n-j)]\}\nonumber \\
&, x(n), \max_{j \in S_n} E[X_j] \biggr\}~, 
\end{align}
where the three terms on the RHS correspond to the decision options sense, access with recall and access with guess, respectively. 

Our goal is to design an online algorithm $\alpha_t, t=1,2,...,T$ based on past observed history $\mathcal F_{t-1},...,\mathcal F_1$, so as to minimize the following strong regret measure,  
\begin{align}
 R_{II}(T)&= \sup_{\mathbf{\pi} \in \Pi^T}  E^{\mathbf{\pi}}\biggl[ \sum_{t=1}^T (X_{\pi_t(\tau)}(t)-\sum_{n=1}^{\tau-1} c_{\pi_t(n)})\biggr]\nonumber \\
 &-  E^{\mathbf{\alpha}}\biggl[ \sum_{t=1}^T (X_{\alpha_t(\tau)}(t)-\sum_{n=1}^{\tau-1} c_{\alpha_t(n)})|\mathcal F_{t-1},...,\mathcal F_0\biggr]~,
\end{align}
where $\pi_t$ is the optimal decision at meta stage $t$ when the information on $\{X_i\}_{i \in \mathcal O}$ is known and $\tau$ the stopping time; $\pi_t(n), n=1,2,...,\tau$ are the channels selected at decision step $n$ of each meta stage $t$. $\alpha_t$ is the decision actually made at $t$ by the user based on past observations when channel statistics is unknown. 

For both problems, if an algorithm can achieve regret $\frac{R_{I}(H)}{H}$ (respectively $\frac{R_{II}(T)}{T}$)  $\rightarrow 0$ then it is called sub-linear in total regret and zero-regret in time average (\rev{optimal asymptotically}). 

\section{Offline Solutions Revisited}\label{sec:off}

To be self-contained as well as to provide certain intuition for the design of the online algorithms, below we present the optimal offline solutions to the scheduling problems in Model I and Model II respectively. 
%
%

\subsection{Algorithm description: Model I}

The solution for the scheduling problem in Model I is surprisingly clean and elegant, and can be easily described as follows. Within each meta stage $l$, the optimal stopping rule is given by a threshold policy \rev{ \cite{Zheng:2009:DOS:1669334.1669354}}: 
 \begin{align}
 \tau^* = \min\{n \geq 1 : X(n) \geq x^*\}~,
 \end{align}
 where $x^*$ is given by the solution for $u$ in the following equation: 
 \begin{align}
 E[X(n)-u]^{+} = \frac{u\cdot \zeta}{p_s\cdot K}~.\label{optimalx}
 \end{align}
The corresponding algorithm is straightforward: at each $n$ when a user needs to make a decision, if $X(n) \geq x^*$, a  user will transmit and otherwise will release the channel. Intuitively this says that when the channel quality is sufficiently good (as compared to $x^*$ which separates the decision regions for stop and continue), a user should transmit. This algorithm will be referred to as \textbf{Offline\_MU} (MultiUser) in our subsequent discussion.

\subsection{Algorithm description: Model II}

The solution for Model II is much more involved;  
this is primarily due to it allowing \emph{access with guess} as an option, which is very different from classical stopping time problems.  In this sense this model presents a generalization. The optimal policy is shown to have three major steps in \cite{Chang07optimalchannel}: parameter calculation, sorting, and decision making, as detailed below.  

\noindent\rule{8.5cm}{1pt}
\textbf{STEP 1: Parameter calculation, $\forall j \in \mathcal O$}
\\[-5pt]
\noindent\rule{8.5cm}{1pt}
\vspace{-15pt}
\begin{align*}
a_j &= \min \{ u: u \geq E[X_j], c_j \geq E[\max (X_j - u, 0)]\}~,\\
b_j &= \max \{ u: u \leq E[X_j], c_j \geq E[\max (u - X_j, 0)]\}~.
\end{align*}
\\[-20pt]
\noindent\rule{8.5cm}{1pt}
~\\
\noindent\rule{8.5cm}{1pt}
\textbf{STEP 2: Channel sorting}
\\[-5pt]
\noindent\rule{8.5cm}{1pt}

1: Initialize $k=1,~S = \mathcal O$.

2: First compute $\mathcal R := \biggl \{ j \in S, a_j = \max_{i \in S} a_i \biggr\}~,$ and then $j^*$:
\begin{align*}
j^* = &\argmax_{j \in \mathcal R}\biggl\{ I_{b_j = a_j}\cdot E[X_j] + I_{a_j > b_j}\cdot\biggl[\nonumber \\
& E[X_j|X_j \geq a_j] - \frac{c_j}{P(X_j \geq a_j)}\biggr]\biggr\}~.
\end{align*}

3: Let $o_k = j^*$ (randomly select one if multiple $j^*$ exists) and set $k:=k+1$. $S = S- \{j^*\}$.
 
4: If $|S| \geq 1$, repeat \textbf{2}; o.w. return the sorted set $\{o_1,...,o_N\}$.

5: Relabel the sorted set as $\{1,2,...,N\}$.

\noindent\rule{8.5cm}{1pt}

\textbf{STEP 1} is based on a threshold property for the optimal policy proved in \cite{Chang07optimalchannel}. Intuitively speaking, $a,b$ separate the decision region as follows. 
A state larger than $a_j$ means further probing is not profitable whereas a state below $b_j$ suggests gain from continued sensing. 
For \textbf{STEP 2} we refer to each of its sub-steps \textbf{m} as \textbf{STEP 2.m} (we will re-use this numbering style in later discussions). The sorting process is straightforward: we start with the full set $\mathcal O$ and at each step we first calculate $\mathcal R$, the set of channels with the highest $a_j$. Then within $\mathcal R$ we further order the channels based on the one-step reward of probing channel $j$ when $x(n)=a_j$ and $j$ being the only remaining channel. The ordering repeats until all channels are in order. 

Given the current information state is $(x(n),S_n)$ at decision epoch $n$ and denoting by $d_s$ the solution to the following equation (solution is guaranteed to exist  \cite{Chang07optimalchannel}):
\begin{eqnarray*}
V(0,S_n) = -c_1 + E[V(\max\{d_s, X_1\}, S_n-\{1\})]~,
\end{eqnarray*} 
\noindent\rule{8.5cm}{1pt}
\textbf{STEP 3 : Decision Making}
\\[-5pt]
\noindent\rule{8.5cm}{1pt}

1: If $x(n) \geq \max_{i \in S_n}a_i$, stop and access the best sensed channel.

2: Otherwise if $x(n) > d_s$, probe the first channel in $S_n$.

3: If $x(n) \leq d_s$ consider the following sub-cases
~\\[-15pt]
\begin{flushleft}
$~~~~~~$(1) : If $b_1 \geq a_2$, then access/guess 1st channel (in $S_n$, w/o sensing).\\
$~~~~~~$(2) : If $b_2 \geq b_1$ or $g_1(0) \geq \max\{E[X_1],g_2(0)\}$, probe 1 in $S_n$.\\
$~~~~~~$(3) : There exists a unique $b_0$, where $b_1 > b_0 > b_2$ and $g_1(b_0) = \max\{E[X_1],g_2(0)\}$. If $x(n) \geq b_0 $ : probe 1st channel. $x(n)< b_0$: guess  channel 1 if $E[X_1] \geq g_2(0)$; probe channel 2 o.w.
\end{flushleft}
~\\[-20pt]
\noindent\rule{8.5cm}{1pt}
where 
$g_i(x) = - c_i + E[V(\max(X_i,x),-i+3)],~i=1,2$, 
and $g_1(x)$ is the expected reward of probing channel 1 facing information state $(x, \{1,2\})$ while $g_2(x)$ is the reward for probing channel 2. 

We denote the algorithm consisting of (\textbf{STEP 1, STEP 2, STEP 3}) as \textbf{Offline\_MC} (MultiChannel) and it serves as the offline benchmark solution for the multichannel scheduling problem. 

\section{Design of Online Algorithms} \label{sec:online}


We detail our online learning algorithm in this section. To generalize the discussion we shall refer to the users in Model I and the channels in Model II as \emph{units}. Then for a unifying framework of the online learning process there are two main phases : \emph{exploration} and \emph{exploitation}  which can be described as follows: 
(1) Exploration: sample the units with sampling times less than $D_1(t) = L \cdot t^z \cdot \log t$ up to meta stage $t$, with $L>0, ~0<z<1$ being constant parameters. Here $L$ is a sufficiently large (we shall specify its bounds later alongside the analysis) \emph{exploration} parameter. When the unit represents a channel, the sampling process is to probe the channel quality; when such an unit represents a user, the process corresponds to letting the user gain access to the channel to gather samples. 
 (2) Exploitation: execute the optimal scheduling policy using collected statistics as detailed in the offline solution, but with built-in tolerance for estimation errors as detailed below.  The sensing results (possibly multiple) from exploitation phases will also be collected and utilized for training purpose.

The above steps are rather standard within the regret learning literature: when a unit has not been explored/sensed sufficiently (e.g., a user has not accessed a channel for sufficient number of times in Model I or a channel has not been sampled sufficiently in Model II), the algorithm enters the exploration phase. Otherwise the algorithm mimics the procedures of calculating the optimal strategies as detailed in the offline solutions but with empirically estimated channel statistics.  One notable difference here is that since the offline dynamic policies involve channel sensing as part of the decision process, effectively additional samples are collected during exploitation phases and used toward estimation.  The general framework of this online approach is summarized as follows. 
\begin{figure}
\noindent\rule{8.5cm}{1pt}
\textbf{Online Solution : A unifying framework}\\[-5pt]
\noindent\rule{8.5cm}{1pt}

1: \emph{Initialization}: Initialize $L,z,t=1$ and sample each unit at least once. Update the collection of sample as $\mathcal F_0$ and the number of samples for each unit $j$ as $n_j(t)$.

2:  \emph{Exploration}: At stage $t$, if $\mathcal E(t): = \{j : n_j(t) < D_1(t) \} \neq \emptyset$, sense the set $\mathcal E(t)$ of units. 

3: \emph{Exploitation}: If $\mathcal E(t) = \emptyset$, calculate the optimal strategy according to steps in the corresponding offline algorithm (with relaxation) based on collected statistics $\{\mathcal F_{\hat{t}}\}_{\hat{t}=1}^{t-1}$. 

4:  \emph{Update}: $t: = t+1$; update sample set and for sampled unit $j$ update $n_j(t) := n_j(t)+1$. 
\\[-5pt]
\noindent\rule{8.5cm}{1pt}
\caption{\textbf{A unifying framework}}
\end{figure}

The exploitation phase is intended for the algorithm to compute and execute the optimal offline strategy using statistics collected during the exploration phase.  However, due to the estimation error, the executed version has to made error tolerant, e.g., by relaxing the 
%
conditions for the steps involving strict equalities. We show how this relaxation is done for the problem in Model II below.  

We now detail the online counterparts for \textbf{Offline\_MU} and \textbf{Offline\_MC} by filling in the details into above general framework. As a notational convention, we will denote by $\tilde{y}$ the estimated version of $y$, $t_j(k)$ the meta stage when the $k$-th sample is  collected for channel $j$ and $E[\tilde{X}]$ the sample mean of $X$. 
\begin{figure}
\noindent\rule{8.5cm}{1pt}
\textbf{Online\_MU : Algorithm details}\\[-5pt]
\noindent\rule{8.5cm}{1pt}

1: \emph{Initialization}: Initialize $L,z,l=1$ and let each user access the channel once. Denote the collected sample for user $j$ at stage $l$ as $\mathcal F^j_l$. Update number of collected samples $n_j(0)=1$ and $\mathcal F^j_0$ .

2:  \emph{Exploration}: At stage $l$, let $\mathcal E(l): = \{j : n_j(l) < D_1(l) \}$. At any decision epoch, if $\mathcal E(l) \neq \emptyset$ and let user $j \in \mathcal E(l)$ transmit right away. If multiple such $j$ exist, a user is selected randomly from $\mathcal E(l)$. 

3: \emph{Exploitation}: Otherwise if $\mathcal E(l) = \emptyset$, calculate the optimal threshold $\tilde{x}^*$ according to Eqn. (\ref{optimalx}) using collected statistics $\{\mathcal F^j_{\hat{l}}\}_{\hat{l}=0}^{l-1}$  for each user $j$ and follow the scheduling strategy detailed in \textbf{Offline\_MU}.

4:  \emph{Update}: $l := l+1$; for user $j$ who accessed the channel update $n_j(l) := n_j(l)+1$ and its sample set $\{\mathcal F^j_{\hat{l}}\}_{\hat{l}=0}^{l}$.
\\[-5pt]
\noindent\rule{8.5cm}{1pt}
\caption{\textbf{Online\_MU}}
\end{figure}

%
%
%
%
%
%

In \textbf{Online\_MC}, besides the clear separation between exploration and exploitation phases, several relaxations are invoked and the relaxation term $\frac{1}{t^{z/2}}$ could be viewed as the tolerance/confidence region. This tolerance region decreases in time $t$ and approaches 0 asymptotically as the estimation errors decrease as well. 
There is an inherent trade-off between exploration and the tolerance region. With more exploration steps (a larger $z$), a finer degree of tolerance region could be achieved. We shall further discuss the roles of $z$ in the analysis. 

\begin{figure}[!ht]
\noindent\rule{8.5cm}{1pt}
\textbf{Online\_MC: Algorithm details}\\[-5pt]
\noindent\rule{8.5cm}{1pt}

1: \emph{Initialization}: Initialize $L,z,t=1$ and sense each channel once. Update the number of channels being sensed and observed as $n_j(t) = 1$. Update the collection of sample for each channel $j$ as $\{\tilde{X}_{j}(t_j(k))\}_{k=1}^{n_j(t)}$. ($\mathcal F_t = \cup_j \{\tilde{X}_{j}(t_j(k))\}_{k=1}^{n_j(t)}$)

2:  \emph{Exploration}: At meta stage $t$, if $\mathcal E(t): = \{j : n_j(t) < D_1(t) \} \neq  \emptyset$, sense the set $\mathcal E(t)$ of channels sequentially and choose the one with best instantaneous condition. 

3: \emph{Exploitation}: If $\mathcal E(t) = \emptyset$ calculate the optimal strategy according to steps in \textbf{Offline\_MC} (with relaxation) based on collected statistics $\{\mathcal F_{\hat{t}}\}_{\hat{t}=0}^{t-1}$ as follows: \begin{itemize}
\item[] \emph{Online.STEP 1}: Calculate $a_j, b_j$ according to the follows 
\\[-15pt]
\begin{align*}
\tilde{a}_j &= \min \{ u: u \geq E[\tilde{X}_j], c_j +\frac{1}{t^{z/2}} \geq E[\max (\tilde{X}_j -u, 0)]\}~, \\[-5pt]
\tilde{b}_j &= \max \{ u: u \leq E[\tilde{X}_j], c_j +\frac{1}{t^{z/2}} \geq E[\max (u - \tilde{X}_j, 0)]\}~.
\end{align*}
~\\[-27pt]
\item[] \emph{Online.STEP 2}: Follow STEP 2 of \textbf{Offline\_MC}  but with the following  relaxation\\[-15pt]
\begin{align*}
\tilde{\mathcal R} = \biggl \{ j \in S, |\tilde{a}_j - \max_{i \in S} \tilde{a}_i| < \frac{1}{t^{z/2}} \biggr\}~.
\end{align*}
~\\[-27pt]
\item[] \emph{Online.STEP 3}: Follow STEP 3 of \textbf{Offline\_MC} but with the following relaxation\\[-15pt]
\begin{align*}
\tilde{b}_1 \geq \tilde{a}_2& - \frac{1}{t^{z/2}}:~\text{(3.3.1)};~\tilde{b}_2 \geq \tilde{b}_1 - \frac{1}{t^{z/2}}:~ \text{(3.3.2)}; \\[-5pt]
\tilde{g}_1(0) &\geq \max\{E[\tilde{X}_1], \tilde{g}_2(0)\}- \frac{1}{t^{z/2}}: ~\text{(3.3.2)}~.
 \end{align*} 
 ~\\[-25pt]
\end{itemize}

4:  \emph{Update}: $t: = t+1$; for sensed channel $j$ update $n_j(t) := n_j(t)+1$ and sample set $\mathcal F_t$. 
\\[-5pt]
\noindent\rule{8.5cm}{1pt}
\caption{\textbf{Online\_MC}}
\end{figure}




\section{Regret analysis}\label{sec:regret}

In this section we analyze performance of the online algorithm. We present the main results for both \textbf{Online\_MU} and \textbf{Online\_MC}.  Since \textbf{Online\_MC} is a much more complex algorithm and its analysis can be easily adapted for \textbf{Online\_MU}, as well as for brevity, we will only provide details for \textbf{Online\_MC}. 


Before formalizing the regret analysis for \textbf{Online\_MC}, we outline the key steps. The regret consists of two parts:  that incurred during exploration phases and that during exploitation phases. For exploration regret, we will try to bound the number of exploration steps that are needed. For the exploitation phase, the regret is determined by how accurate decisions are made using estimated values. Specifically, \emph{Online.STEP 1} does not have a decision making step as it is simply a calculation, though we will show later in the proof the calculation of $\{a_j,b_j\}_{j \in \mathcal O}$ does play an important role in the sorting and decision making process. In \emph{Online.STEP 2} if the sorting is done incorrectly then this could lead to error in \emph{Online.STEP 3}.  
as all decision making and sensing orders are based upon the ordering of the channels. \emph{Online.STEP 3} has the following error: (1) error in  the calculation of $a_j, b_j$s, (2) error in calculating $d_s$, and (3) error in calculating a set of value functions for sub-step 3.3.

\subsection{Assumptions}

We state a few mild technical assumptions. We will assume non-trivial channels, i.e., $E[X_j] > 0, \forall j \in \mathcal O$, so that they all have positive average rates. We will also assume all channel realizations are bounded, i.e., finite support over all channel condition,
$
0 \leq \sup_{j \in \mathcal O,\omega} X_j(t;\omega) < \infty,~\forall t, 
$
$\omega$ being an arbitrary channel realization. 
This is not a restrictive assumption since in reality the transmission rate is almost always non-trivial and bounded. 

Moreover denote 
\begin{align*}
\Delta^* = \max_{t, i\neq j, \omega_i, ~\omega_j}|X_j(t;\omega_j) - X_i(t;\omega_i)| + \sum_{i \in \mathcal O} c_i~.
\end{align*}
$\Delta^*$ can be viewed as an upper bound for a one step loss when a sub-optimal decision is made and $\Delta^* < +\infty$ (note $c_i$s are finite).
Finally, we assume the cdf of each channel $i$'s condition satisfies the Lipschitz condition, i.e., there exists $\mathcal L (\text{note different from $L$}), \alpha > 0$ such that 
\begin{align*}
|F_{X_i}(x+\delta) - F_{X_i}(x)| \leq \mathcal L\cdot |\delta|^{\alpha}, \forall i,x,\delta~.
\end{align*}
%
%
The Lipschitz condition has been observed to hold for various distributions, for example the exponential distribution and uniform distribution \cite{heinonen2005lectures}.

\subsection{Main results for Online\_MC}
We first separate the regret for different phases. We have the following simple upper bound on the regret $R_{II}(T)$,
\begin{align*}
R_{II}(&T) = R_e(T) + R_s(T)\leq R_e(T) + \biggl( R_2(T) + R_3(T)\biggr)~.\label{bound:sep}
\end{align*}
The first term $R_e(T)$ is the regret from exploration phases. $R_s(T)$ is the regret from exploitation which could be further upper bounded by the two terms from \emph{Online.STEP 2\&3} of \textbf{Online\_MC} respectively:  $R_2(T)$ comes from the sorting procedure and $R_3(T)$ comes from the last step of decision making. Notice for  \emph{Online.STEP 1} there is no direct regret incurred by parameter calculation: the errors in the calculation are reflected in \emph{Online.STEP 2\&3} later. The idea of upper bounding the regret by a union bound will be repeatedly utilized in the following analysis. For example, we can show that the regret in each step above can again be upper bounded by the sum of regrets of each of its sub-steps. Therefore we will not restate the details of the bounding for the rest of the proof. Denote the sum 
$s_{p}(T) := \sum_{t=1}^T \frac{1}{t^p}$. We have our main result for the regret analysis summarized as follows.

\begin{theorem}
There exists a constant $L$ such that the regret for \textbf{Online\_MC} is bounded by 
\begin{align*}
R_{II}&(T) \leq \Delta^* \biggl\{ N LT^z \log T + 
C_1\cdot s_{\alpha\cdot z/2}(T)+ C_2\cdot s_2(T)\biggr\}~,
\end{align*}
time uniformly, where $C_1,C_2>0$ are constants.
\end{theorem}

Here $L$ is larger than a certain positive constant which we detail later. It is easy to notice since $T^z \log T$ and $s_{\alpha\cdot z/2}$ are both sub-linear terms ($s_{\alpha\cdot z/2}$ is on the order of $1-\frac{\alpha\cdot z}{2}$ while $s_2(\cdot)$ is bounded by a constant since $s_p(T) < \infty, \forall p > 1, T$.), $R_{II}(T)$ is also sub-linear and asymptotically we achieve zero-regret on average ($\lim_{T\rightarrow \infty}R_{II}(T)/T=0$). The first term $T^z\log T$ is due to the exploration while the term $s_{\alpha\cdot z/2}$ comes from exploitation. Clearly we see with a larger $z$ (more exploration invoked), we will have a larger regret term from exploration phases; however the regret for exploitation will decrease. The balanced setting is achieved at 
$
z=1-\frac{\alpha\cdot z}{2} \Rightarrow z = \frac{2}{2+\alpha}~.
$



\subsection{Bounding exploration regret}

We start with bounding the exploration regret $R_e(T)$. 
\begin{theorem}
The exploration regret $R_e(T)$ is bounded as
\begin{align}
R_e(T) \leq D_1(T)\cdot N\Delta^*~.
\end{align}
\end{theorem}
\begin{proof}
 Notice since the exploration phase requires $D_1(T)$ samplings for each channel up to time $T$, we know there are at most $N\cdot D_1(T)$ exploration phases being triggered. For each exploration phase, the regret is bounded by $\Delta^*$, completing the proof.
\end{proof}


\subsection{Bounding exploitation regret}

We next consider regret incurred during exploitation phases. 

\subsubsection{Exploitation regret for \emph{Online.STEP 2}}

We bound the regret associated with the sorting process of \emph{Online.STEP 2}. Details can be found in the Appendix.  

\begin{lemma}
Regret $R_2(T)$ is bounded as follows,
\begin{align*}
R_2(T) \leq \Delta^* \cdot 2\cdot N\sum_{t=1}^T \frac{2}{t^2}~.
\end{align*}
\label{lemmar2}
\end{lemma}

\rev{The main challenge in this proof is to relate the sampling uncertain to the ones in our decision making process. First of all we could show the calculation of $E[X_j], a_j, b_j$ can fall into certain confidence region when the number of exploration steps are large enough ($L$). Moreover the estimation errors of $a_j,b_j$ are proportional to the one for $E[X_j]$. Intuitively this is due to the calculation of $a_j, b_j$ which relates to the calculation of $E[X_j]$ in a piece-wise linear way. Next consider calculating $j^*$. There are potentially two types of errors. First is the decision error associated with the decision process of telling whether the following holds $
I_{\tilde{a}_j = \tilde{b}_j}.~
$
To bound the error of making the wrong call, we are going to show when $a_j = b_j$, we could bound the probability of $\tilde{a}_j \neq \tilde{b}_j$. Alongside the binary decision making, we also have the estimation error for \emph{Online.STEP 2} for terms such as $E[X_j], E[X_j|X_j \geq a_j], \frac{c_j}{P(X_j \geq a_j)}$.}
%
%
%


\subsubsection{Exploitation regret for \emph{Online.STEP 3}}

We bound the regret associated with the decision making step (\emph{Online.STEP 3}). Details can be found in the appendix. 
\begin{lemma}
Regret $R_3(T)$ is bounded as follows, 
\begin{align*}
R_3(T) &\leq \Delta^* \cdot \biggl(C_1\cdot s_{\alpha\cdot z/2}(T) + C^{*}_2\cdot s_2(T)\biggr)~,
\end{align*}
where $C_1, C^{*}_2$ are positive constants.\label{lemmar3}
\end{lemma}

\rev{The proof is obtained by bounding the decision errors in each of the sub-steps \emph{Online.STEP 3.1, 3.2, 3.3}. The technical challenges again come from bounding the errors with calculating various parameters in the decision making steps, including for instance $d_s$ and the value function $V(,\cdot,)$s. }

Combine $R_e(T),R_2(T),R_3(T)$ we have our main results.

%


\subsection{Discussion on parameter $L$}

In most of our proved results, we assumed $L$ to be significantly large. We summarize the actual conditions on $L$ below (please refer to the appendix for details): 
\begin{align*}
\text{\bf (Condition 1) : }L  &\geq \max\{4,1/(\frac{\min_{a_{k_1} \neq a_{k_2}} |a_{k_1} - a_{k_2}|}{
2\max_{j} c_{1,j}})^2\}~,\\
\text{\bf (Condition 2) : }L  &\geq 1/\epsilon^2_o~,
\end{align*}
where $\{c_{1,j}\}_{j \in \mathcal O}$ is a set of positive constants and $\epsilon_o$ is a solution of $\epsilon$ for 
$
C \cdot (\epsilon + \mathcal L \cdot(c_{1,j}+1)^{\alpha}\epsilon^{\alpha}) \leq \frac{\min \{\epsilon_3,\epsilon_4\}}{2}~,
$
where $C$ is a positive constant and
$
\epsilon_3 = \min_{j \neq k} |E[X_j] - E[X_k]|, ~\epsilon_4 = \min_{j \neq k} |\frac{E[X_j]-c_j}{P(X \geq a_j)}-\frac{E[X_k]-c_k}{P(X \geq a_k)}|~
$ (we assume $\epsilon_3,\epsilon_4 > 0$).

From \textbf{(Condition 1)} we know when $\{a_i\}$s are closer to each other, $L$ should be chosen to be larger. Also from \textbf{(Condition 2)} we know when channels' expected reward $E[X_j]$ and $\frac{E[X_j]-c_j}{P(X \geq a_j)}$ (can be viewed as potential term when sensed) are closer to each other, again $L$ should be chosen to be larger. The intuition here is that in such cases a larger $L$ can help achieve higher accuracy for the estimations to differentiate two channels that are similar.


The selection of $L$ depends on a set of $\epsilon$s, which further depends on statistical information of $X_j$s (though weaker as we only need to know a lower bound of them) which is assumed to be unknown. However, following a common technique  \cite{agrawal1995continuum}, the assumption can be further released but with potentially larger regret.  
In particular one can show that at any time $t$ with $L$ being a positive constant the estimation error $\epsilon_t$ for any terms (e.g., $a,b$ or $E[X_j]$) satisfy the following,
$
P(\epsilon_t >  \frac{1}{t^{\theta}}) \leq \frac{1}{t^{\nu}} ~,
$
with $\theta,\nu > 1$. 
Therefore with the error region $\epsilon_t$ being small enough, there would be no error associated with differentiating the channels of the algorithm. Thus there exists a constant $T_0$ such that,
$
\epsilon_t < \min \epsilon, \forall t > T_0.~
$Consider the the case $\epsilon_t \leq  \frac{1}{t^{\theta}}$. Since when the error happens under this case,  two estimated terms (the sub-optimal and optimal one) are separated by at most $2 \epsilon_t$. The probability of the corresponding term falls into this region is bounded as
$
|F_{X_i}(x+\epsilon_t)-F_{X_i}(x-\epsilon_t)|  \leq \mathcal L\cdot2^{\alpha}\cdot  \frac{1}{t^{\alpha \cdot \theta}}
$ 
by the Lipschitz condition. Therefore we have the extra error bounded by 
$
\sum_{t=1}^{T_0}  \mathcal L \cdot2^{\alpha}\cdot \frac{1}{t^{\alpha \cdot \theta}}
$, which is a constant growing sub-linearly up to time $T_0$.

\subsection{Main results for Online\_MU}

For \textbf{Online\_MU} we can similarly prove the following result 
\begin{theorem}
There exists a constant $L$ such that the regret for \textbf{Online\_MU} is bounded by 
\begin{align*}
R_{I}&(H) \leq \Delta^* \biggl\{ M LH^z \log H + \hat{C}_1\cdot s_{\alpha\cdot z/2}(H)+ \hat{C}_2\cdot s_2(H)\biggr\}~,
\end{align*}
time uniformly, where $\hat{C}_1,\hat{C}_2>0$ are constants.
\end{theorem}
Notice though $R_{I}(H)$ looks similar to $R_{II}(T)$, they may have very different parameters for each term, i.e, $\hat{C}_1,\hat{C}_2$ may be quite different from $C_1,C_2$, as well as different constraints for $L$ due to the different statistical structure of the two problems. Again the first term is coming from exploration phases, the second term due to inaccurate calculations of $x^*$ and last term bounds the event that $\tilde{x}^*$ is too different from $x^*$.

\section{Simulation}\label{sec:simulation}

In this section we show a few examples of the performance of the proposed online algorithm via simulation. We measure the average \emph{regret} rate $R_I(l)/l(R_{II}(t)/t)$ 
and compare our performance to the optimal offline algorithm, a static best single channel policy, as well as that of a weak-regret algorithm. 

  For simplicity of demonstration we assume channel qualities follow \emph{exponential distribution} but with different parameters \footnote{We have similar observations for other distributions. The details are omitted for brevity.}. The corresponding distributions' parameters are generated uniformly and randomly between $[0,0.5]$. Users' attempt rate $p_i$s are uniformly generated in the interval $[0,0.5]$ (in Model I). The costs for sensing the channels (in Model II) are also randomly generated according to uniform distribution between $[0,0.1]$. In the following simulation for Model I we have $M=5$ users while for Model II we have $N = 5$ channels. Simulation cycle is set to be $H = T = 4,000$. In the set of results for performance comparison with offline solutions, we set the exploration parameters as $L=10, z=1/5$. Later on we show the performance comparison w.r.t. different selection of $L$ and $z$.
%
\subsection{Comparison with Offline Solution}

We first take the difference between the oracle (\textbf{Offline\_MU}) and \textbf{Online\_MU} at each step $t$ and divide it by $t$ (i.e., we plot $R_{I}(t)/t$). This regret rate is plotted in Figure~\ref{err:1} and clearly we see a sub-linear convergence rate. We repeat the experiment for \textbf{Online\_MC} and the regret convergence is shown in Figure~\ref{err:2}, which validate our analytical results. To make the comparison more convincing, we compare the accumulated reward between \textbf{Online\_MC}, \textbf{Offline\_MC} and the best single-channel (action) policy, which always selects the best channel in terms of its average rate (channel statistics is assumed to be known a priori) in Figure~\ref{test:1}. 
In particular we see the accumulated rewards of \textbf{Online\_MC} (red square) is close to the performance of the oracle (blue circle) who has all channel statistical information and follows the optimal decision  process as we previously depicted in \textbf{Offline\_MC}. We observe the dynamic policies clearly outperform the best single channel policy.

\begin{figure}[!h]
\centering
\subfigure{
                \centering
                \includegraphics[width=0.45\textwidth]{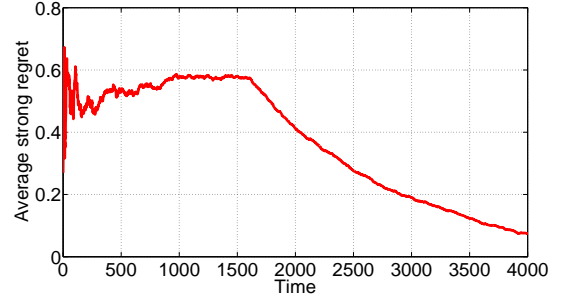}
        }\caption{Convergence of average regret : \textbf{Online\_MU}}\label{err:1}
\end{figure}
\begin{figure}[!h]
\centering
\subfigure{
                \centering
                \includegraphics[width=0.45\textwidth]{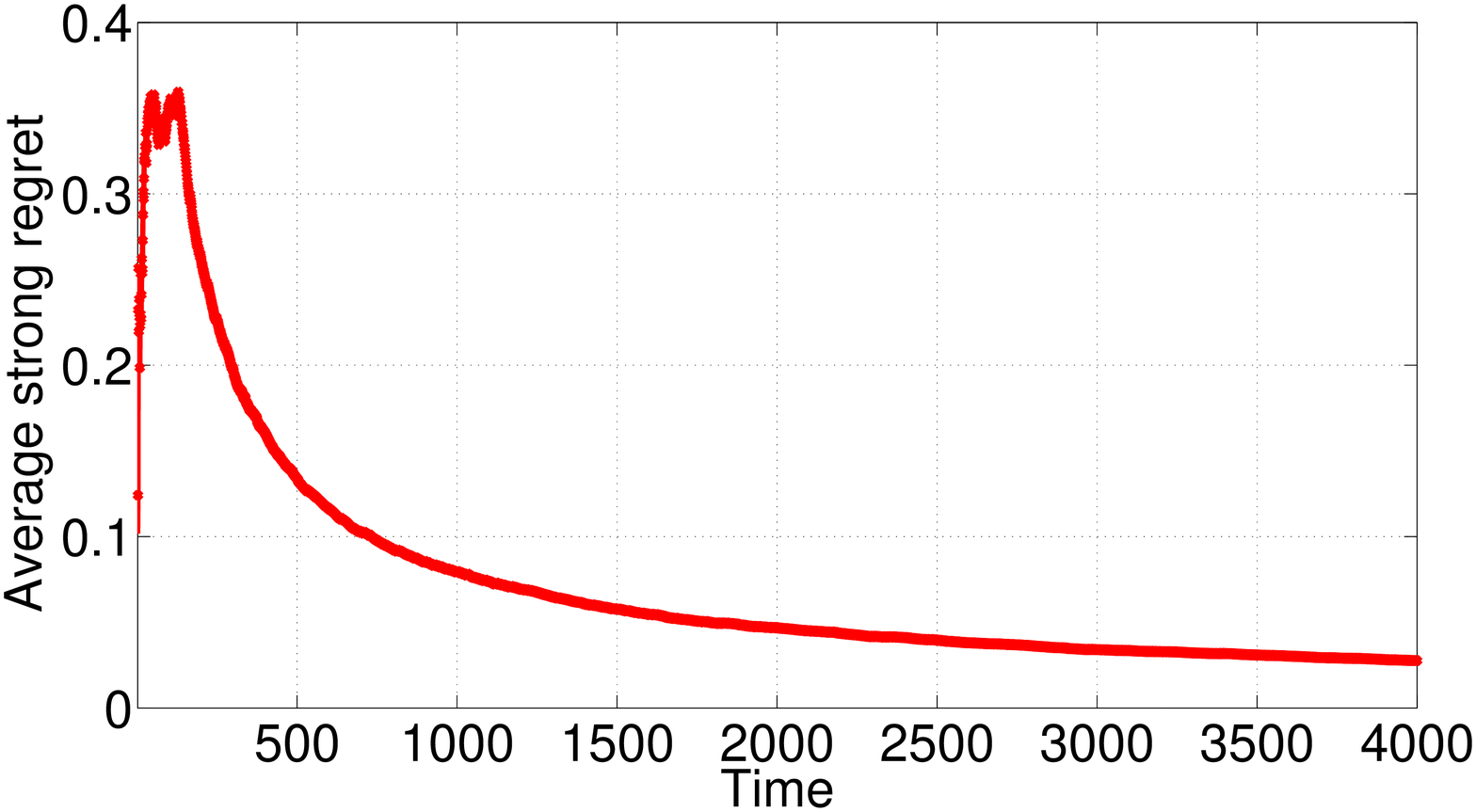}
        }\caption{Convergence of average regret : \textbf{Online\_MC}}\label{err:2}
\end{figure}

\begin{figure}[!ht]
\centering
\subfigure{
                \centering
                \includegraphics[width=0.45\textwidth]{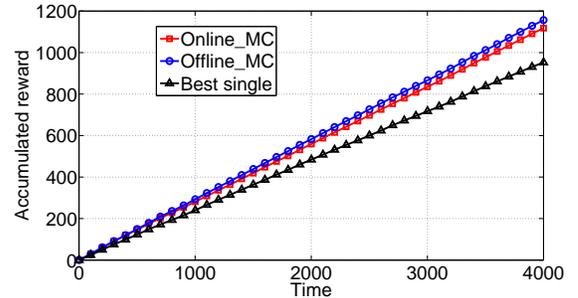}}
                \caption{ \textbf{Online\_MC} v.s. \textbf{Offline\_MC} v.s. Best single}\label{test:1}
\end{figure}

\newpage
\subsection{\rev{Comparison with naive reinforcement learning solution}}
\rev{
As we mentioned earlier in the introduction, there exist online solutions for a user to find the best channel in terms of its average condition (minimizing weak regret). We demonstrate the advantages of our proposed online algorithm with a comparison between \textbf{Online\_MC} with UCB1, a classical online learning weak-regret algorithm \cite{Auer:2002:FAM:599614.599677} suitably designed for IID bandits. The result in Figure \ref{test:2}  clearly shows the performance gain by using \textbf{Online\_MC}.
\begin{figure}[!ht]
\centering
\subfigure{
                \centering
                \includegraphics[width=0.45\textwidth]{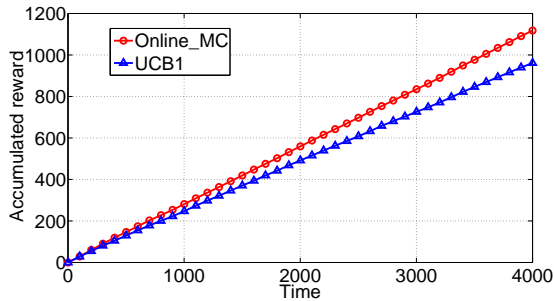}}
                \caption{ \textbf{Online\_MC} v.s. UCB1 }\label{test:2}
\end{figure}
}
\subsection{Effects of parameter selection}

We next take a closer look at the effects of parameter selection, primarily with $L$ and $z$. We demonstrate with \textbf{Online\_MC}. We repeat the above sets of experiment w/ different $L, z$ combinations and tabulate the average reward per time step. From Table \ref{comp:parameterL} we observe the selection of $L$ is not monotonic: a smaller $L$ incurs less exploration steps but more errors will be invoked at exploitation steps due to its less confidence in calculating the optimal strategy. On the other hand, a large $L$ inevitably imposes higher burden on sampling and thus becomes less and less favorable with its increase. 
Similar observations hold for $z$ since $z$ controls the length of exploration phases jointly with $L$ but with different scale. However it is indeed interesting to observe that when $z$ grows large enough (e.g., $z=\frac{1}{2}$), the performance drops drastically: this is due to the fact in such a case more than enough efforts have been spent in sensing steps.

\begin{table}[!h]
\begin{scriptsize}
\begin{center}
\begin{tabular}{  | c| c | c | c | c  | c|}
    \hline
    $L (z=1/5)$ & 5 & 10 & 20 & 30 & 40 \\ \hline \hline
    Average reward & 0.3391 & \textbf{0.3522} &  0.3353 & 0.3183& 0.3166  \\ \hline 
  \end{tabular}
\end{center}
  \caption{Diff. $L$ (Avg. = 0.27 w/ random channel selection)}
  \label{comp:parameterL}
\end{scriptsize}
\end{table}

\begin{table}[!h]
\begin{scriptsize}
\begin{center}
\begin{tabular}{  | c| c | c | c | c  | c|}
    \hline
    $z (L=10)$ &1/6 & 1/5 & 1/4 & 1/3 & 1/2  \\ \hline \hline
    Average reward & 0.3411 &0.3522 & \textbf{0.3557} & 0.3017 & 0.1949   \\ \hline 
  \end{tabular}
\end{center}
  \caption{Diff. $z$ ( Avg. =0.27 w/ random channel selection)}
  \label{comp:parameterz}
\end{scriptsize}
\end{table}


%
%

\section{Discussion}\label{sec:discuss}


\rev{In this section we discuss several possible extensions of the current sets of results, primarily concerning the statistical assumption of channel evolutions. Throughout the paper, we assume the channel statistics over time evolves as an IID process, though with unknown distributions and parameters. An immediate extension of this work is to study the online learning algorithm when such evolution is Markovian. For Markovian channels we need to again consider two categories of problems, namely rested and restless bandits \cite{mahajan2008multi}.  For rested bandit, the offline (when transition parameters being known) optimal solution is famously known as the Whittle's index. Following similar exploration and exploitation procedures detailed in the current paper we can achieve an accurate enough estimation of all transition parameters of the bandits and thus approximate the optimal indices.

The main difficulties for restless case are due to the facts that even the offline strategy is not easy to obtain under this scenario, that is we do not have a clear target to track. Under certain setting, myopic policy has been shown to be optimal in one of our work \cite{liu:it14} and following procedures in RCA proposed in \cite{tekin2011online} for learning with restless bandit we could again achieve an fairly accurate estimation and approach myopic sensing in an online fashion. However optimal solution for general stopping rule/sequential decision making problems with restless bandits is not yet clear at this moment, which is also one of our focus.

Another interesting extension we would like to approach is to learn with (multiuser) interferences. A natural way of doing this is to combine stochastic bandit learning (for channels availability) with adversarial learning (for users interference). We conjecture similar results could be obtained while we emphasis in such case two types of explorations would be needed: first is the exploration for other users' availability as commonly done in adversarial settings and the other one for exploring channels' statistics. However the technical validation would NOT be trivial to detail out since considering multiuser effects in a sequential decision making process is known to be hard, even in a offline setting \cite{liu2013stay}, primarily due to collision and interferences.


The third aspect we concern is on the assumption that within the time horizon of our problem the statistical properties of channels stay unchanged. However though we made such assumption (in order to derive bounds), the exploration nature of the learning algorithm in principle is designed to detect and adapt to changes in the statistics. We are currently looking into the problem of using additional randomization techniques to enhance the adaptivity. Notably one of recent paper proved a sharp bound (sub-linear) for certain cases when such non-stationary statistical properties satisfy bounded variation \cite{besbes2014optimal}. The technical difficulties in our setting are naturally more challenging since we not only need to track the change of each bandit's mean reward, but also many other statistical parameters that are associated with the decision making processes.


}

\section{Conclusion}\label{sec:conclude}

In this paper we studied online channel sensing and transmission scheduling in wireless networks when channel statistics are unknown a priori. Without knowing such information we propose an online learning algorithm which helps collect samples of channel realization while making optimal scheduling decisions. 
We show our proposed learning algorithm (for both a multiuser and multichannel model) achieves sub-linear regret uniform in time, which further gives us a zero-regret algorithm on average. Our claim is validated via both analytical and simulation results.

\section*{Acknowledgment}

This work is partially supported by the NSF under grant CNS 1217689.

\bibliographystyle{unsrt}
\bibliography{myrefs}
%
\section*{APPENDICES}
\subsection*{Notations}
We summarize the main notations in Table \ref{notation}.


\begin{table}
\begin{center}
  \begin{tabular}{| c|| c | }
    \hline
    Notations & Physical meaning \\ \hline\hline
       $M/\mathcal M$ &number/set of users\\\hline
     $N/\mathcal O$ & number/set of channels \\ \hline
    $S$ & subset of channels\\ \hline
    $X_i (X_i(t))$ & channel $i$'s reward (at time $t$) \\ \hline
    $c_i$ & cost for sensing channel $i$\\ \hline
        $p_i$ & access attempt rate of user $i$ \\ \hline
    $V(x,S)$ & value function with state $(x,S)$ \\ \hline
    $f_{X_i}, F_{X_i}$ & p.d.f./c.d.f. of channel $i$\\ \hline
    $\pi,\alpha$ & access \& sensing policies \\ \hline
    $t,n$ & system time, decision step for each $t$ \\ \hline
    $R_{I(II)}(H(T))$ & accumulated regret up to stage $H(T)$ \\ \hline
    $(x(n),S_n)$ & information state at $n$-th epoch \\ \hline
    $L,z$ & exploration parameters \\ \hline
    $\mathcal L,\alpha$ & Lipschitz parameters \\ \hline
\end{tabular}
\end{center}
\caption{Main Notations}\label{notation}
\end{table}

\subsection*{Outline of the proofs and main results}
Due to space limitation we first sketch the main steps and results towards establishing the proved theorems. 

\subsubsection*{ Proof of regret for $R_2(T)$}

\begin{lemma}
With sufficiently large $L(\geq \frac{1}{\epsilon^2})$, $\forall j$ we have, 
$
P(|E[\tilde{X}_j] - E[X_j]| > \epsilon) \leq \frac{2}{t^2}~,$ and $
P(|\tilde{a}_j - a_j| > c_{1,j}\cdot \epsilon) \leq \frac{2}{t^2},~P(|\tilde{b}_j - b_j| > c_{2,j}\cdot \epsilon) \leq \frac{2}{t^2}~,$
where $\epsilon, c_{1,j},c_{2,j}$ are positive constants.\label{lemma:ab}
\end{lemma}

Based on above results we could show
\begin{lemma}
At time $t$ with sufficiently large $L$, and any iteration steps of the sorting procedure of Online.STEP 2 we have 
\begin{align*}
P(\mathcal R \neq \tilde{\mathcal R}) \leq N\cdot \frac{2}{t^2}~.
\end{align*}
\label{bound:r}
\end{lemma}
~\\[-35pt]


Consider calculating $j^*$ we have the following results,
\begin{lemma}
At time $ t$ with sufficiently large $L$, the error for sorting set $S$ is bounded as,
\begin{align*}
P(\tilde{S} \neq S) \leq N \cdot \frac{2}{t^2}~.
\end{align*}
\label{bound:s}
\end{lemma}
~\\[-30pt]

Putting up all terms and multiple by $\Delta^*$ we have results claimed in Lemma \ref{lemmar2}.

\subsubsection*{Proof of regret for $R_3(T)$}

We sketch the key steps towards getting the claim.

\paragraph{Online.STEP 3.1}


At first step of deciding whether $x(n) \geq a_1$ of \emph{Online.STEP 3}, there will be no error when $x \leq \min \{\tilde{a}_1, a_1\}$ or $x \geq \max \{\tilde{a}_1, a_1\}$. Consider $x$ falling in the middle. Make $\epsilon$ being small enough,
$
\epsilon = \frac{1}{t^{z/2}}~.
$
As we already proved
$
P(|\tilde{a}_1-a_1| >  c_{1,1}\cdot \epsilon) < \frac{2}{t^{2}}.~
$Also due to the relaxation of $\mathcal R$, the difference between $\tilde{a}_1$ and the true $a_1$ is bounded away by at most $c_{1,1}\cdot \epsilon+\epsilon$. For $|\tilde{a}_1-a_1| \leq (c_{1,1}+1)\cdot \epsilon$, the probability that $x$ falls within the middle is bounded as  
\begin{align*}
&P(\exists i ~\st~ X_i(t) \in [\min \{\tilde{a}_1, a_1\}, \max \{\tilde{a}_1, a_1\}]) \nonumber \\
&\leq \sum_i P(X_i(t) \in [\min \{\tilde{a}_1, a_1\}, \max \{\tilde{a}_1, a_1\}])\nonumber \\
&\leq N\cdot |F_{X_i}(\tilde{a}_1) - F_{X_i}(a_1)| \leq  \frac{N\mathcal L\cdot (c_{1,1}+1)^{\alpha}}{t^{\alpha\cdot z/2}}~,
\end{align*}
by Lipschitz condition. Add up for all $t$ we have a sub-linear term. 

\paragraph{Online.STEP 3.2}

We first prove the following results.
\begin{lemma}
With sufficiently large $L$ and information state $(x,S)$, we have at time $t$ $\forall \epsilon>0$
\begin{align*}
P(|\tilde{V}(x,\tilde{S}) - V(x,S)| > |S| \cdot \epsilon) \leq \frac{2}{t^2}~.
\end{align*} 
\label{lemma:vf}
\end{lemma}
~\\[-40pt]

Based on above results we prove that the estimation of $d_s$ can be bounded by a confidence region, which we detail as follows. 
\begin{lemma}
With sufficiently large $L$ and channel set $S$
\begin{align*}
P(|\tilde{d}_s - d_s| > \frac{2 |S|+3}{C_{d_s}} \cdot \epsilon) \leq \frac{4}{t^2}~,
\end{align*}
at time step $t, \forall \epsilon>0$, where $C_{d_s} =  P(X_1 \leq d_s)$.\label{theorem:ds}
\end{lemma}
(Sketch) The proof is primarily done via analyzing the estimation errors from both sides of the equation 
$$
V(0,S_n) = -c_1 + E[V(\max\{d_s, X_1\}, S_n-\{1\})]~,
$$ 
which decides $d_s$. For bounding the value functions we repeatedly use Lemma \ref{lemma:vf}. Taking $L\geq 4$ and $\epsilon=\frac{1}{t^{z/2}}$ will lead to our bounds. 

\begin{remark}
The above result invokes a constant $C_{d_s} = P(X_1 \leq d_s)$. If $P(X_1 \leq d_s) = 0$, i.e., $X_1(\omega) >d_s, \forall \omega$ our bound is not well defined. In fact under this case, what really matters is the overlapping between $[0,\tilde{d}_s]$ and $[\underline{X}_j, \overline{X}_j]$ (support of $X_j$). So long as the overlapping is bounded small enough, the decision error is again bounded. 
\end{remark}

\paragraph{Online.STEP 3.3}


When $x(n) < d_s$, the optimal decision comes from one of three cases. For the first two cases, we have the following lemmas characterizing the regrets :
for sub-steps \emph{Online.STEP 3.3.1, 3.3.2} there are possibly three decisions to make and we have their error bounded as follows (detailed proofs omitted)
\begin{lemma}
With sufficiently large $L$, (1). if $b_1 \geq a_2 $,
$
P(\tilde{b}_1 < \tilde{a}_2 - \frac{1}{t^{z/2}}) \leq \frac{2}{t^2}~.$
(2). If $b_2 \geq b_1 $, 
$
P(\tilde{b}_2 < \tilde{b}_1 - \frac{1}{t^{z/2}}) \leq \frac{2}{t^2}~.
$
(3). If $g_1(0) \geq \max\{E[X_1], g_2(0)\} $, 
$
P(\tilde{g}_1(0)  <  \max\{E[\tilde{X}_1], \tilde{g}_2(0)\}  - \frac{2}{t^{z/2}}) \leq \frac{2}{t^2}~.
$
\end{lemma}

(Sketch) For error in $b_1$ in \emph{Online.STEP 3.3.2}, the analysis is the same as for $a_1$ as in \emph{Online.STEP 3.1} since we already established its estimation error bounds.


For the last case in \emph{Online.STEP 3.3.3}, first notice if $E[X_1] = g_2(0)$, there is no error associated with the last step since guess (access w/o sensing) the first channel and probe the second essentially return the same expected reward. Therefore we show the error analysis when $E[X_1] \neq g_2(0)$.  We then bound the error of estimating $b_0$ (this is similar with proving the bound for $d_s$ and we omit the details for proof) : with $C_{b_0}$ being certain constant,
$
P(|\tilde{b}_0 - b_0| > \frac{2\epsilon}{C_{b_0}}) \leq \frac{2}{t^2}~.
$
Moreover we have the following results: (Details for proof omitted as it is quite similar to previous ones.)
At time $t$ 
$
P(\text{\bf sign}(E[\tilde{X}_1] - \tilde{g}_2(0)) \neq \text{\bf sign}(E[X_1] - g_2(0))) \leq \frac{2}{t^{2}}.~
$ These cover all parameters needed for the decision making queries.
~~\\
Putting up all terms we have results claimed in Lemma \ref{lemmar3}.

\newpage

\section*{ Proof for Lemma \ref{lemma:ab}}
~\\
\begin{proof}
First of all by law of large numbers with enough sampling we could bound the different $|E[\tilde{X}] - E[X]|$ by a positive constant $\epsilon$. Specifically by Chernoff-Hoeffding bounds we have
\begin{align*}
P(|E[\tilde{X}] - E[X]| > \epsilon) \leq 2\cdot e^{-2\cdot \epsilon^2\cdot L\cdot t^z \cdot \log t}~,
\end{align*}
so long as $L \cdot t^z \geq \frac{1}{\epsilon^2}$ we have the results. 


The rest of the proof can be done by proving contradictions. First let us assume $\tilde{a}_j > a_j + \epsilon$. Since $(X_j - \mu)^+$ and $(\mu-X_j)^+$ are also i.i.d. for any constant $\mu$, we know
\begin{align*}
P(|E[(\tilde{X}_j-\mu)^+] - E[(X_j-\mu)^+]| > \epsilon) \leq \frac{2}{t^2} ~,\\
P(|E[(\mu-\tilde{X}_j)^+] - E[(\mu-X_j)^+]| > \epsilon) \leq \frac{2}{t^2}~.
\end{align*}

Now consider the case with $E[(\tilde{X}_j-\mu)^+] - E[(X_j-\mu)^+]| \leq \epsilon$. Then we have
\begin{align*}
&E[(\tilde{X}_j-\tilde{a}_j)^+] \leq E[(\tilde{X}_j-a_j - c_{1,j}\epsilon)^+] \nonumber \\
&\leq E[(\tilde{X}_j-a_j )^+] - (1-P(a_j \leq X_j \leq a_j+c_{1,j}\epsilon)) c_{1,j}\epsilon \nonumber \\
&\leq  E[(\tilde{X}_j-a_j )^+] - (1-P(a_j \leq X_j \leq a_j+c_{1,j}\epsilon)) c_{1,j}\epsilon + \epsilon \nonumber \\
&\leq E[(X_j-a_j)^+] \leq c_j~.
\end{align*}
So as long as we make sure,
\begin{align*}
(1-P(a_j \leq X_j \leq a_j+c_{1,j}\epsilon)) c_{1,j} \geq 1~,
\end{align*}
i.e., when
$
c_{1,j} \geq \frac{1}{1-P(a_j \leq X_j \leq a_j+c_{1,j}\epsilon)}~,
$
we have the above holds which contradicting the optimality of $\tilde{a}_j$.

Consider the case when $\tilde{a}_j < a_j -c_{1,j}\epsilon$, similarly we could prove that with an appropriately chosen $c_{1,j}$ we have
\begin{align*}
\tilde{a}_j \geq E[X_j] -c_{1,j} \epsilon~,
\end{align*}
i.e., $\tilde{a_j} + c_{1,j}\epsilon \geq E[X_j]$. And moreover
\begin{align*}
E[(\tilde{X}_j-(\tilde{a}_j+c_{1,j}\epsilon))^+] \leq E[(X_j-a_j)^+] \leq c_j ~,
 \end{align*} 
which contradicts the optimality of $a_j$. The proof for $b_j$ is similar with $a_j$ and we omit the details for a concise presentation.
\end{proof}
\section*{Proof of Lemma \ref{bound:r}}
~\\
\begin{proof}
  First we have as long as 
$$
\max_{j} c_{1,j}\cdot \epsilon < \frac{\min_{a_{k_1} \neq a_{k_2}} |a_{k_1} - a_{k_2}|-\frac{1}{t^{z/2}}}{2}
$$
there will be no error with sorting $a$s. To see this if $a_j > a_k$ we have
$
\tilde{a}_j - \tilde{a}_k \geq a_j - a_k -  c_{1,j}\cdot \epsilon - c_{1,k}\cdot \epsilon a_j - a_k > 1/t^{z/2}.~
$Since 
\begin{align}
P(|E[\tilde{X}_j]-&E[X_j]| > \frac{\min_{a_{k_1} \neq a_{k_2}} |a_{k_1} - a_{k_2}|-\frac{1}{t^{z/2}}}{
2\max_{j} c_{1,j}} )\nonumber \\
&\leq  2\cdot e^{-2\cdot (\frac{\min_{a_{k_1} \neq a_{k_2}} |a_{k_1} - a_{k_2}|-\frac{1}{t^{z/2}}}{
2\max_{j} c_{1,j}})^2 \cdot L\cdot t^z\log t }~,
\end{align}
by Chernoff-Hoeffding bound. Therefore if we have roughly (since $\frac{1}{t^{z/2}}$ is a much smaller term in order )
\begin{align}
L \cdot t^z\geq 1/(\frac{\min_{a_{k_1} \neq a_{k_2}} |a_{k_1} - a_{k_2}|}{
2\max_{j} c_{1,j}})^2~.
\end{align}
 a $O(1/t^2)$ error is guaranteed. For $a_j = a_k$ we can similarly bound probability that $|\tilde{a}_j - \tilde{a}_k| > 1/t^{z/2}$ as long as $L \geq 4$.
\end{proof}

\section*{ Proof of Lemma \ref{bound:s}}
~\\
\begin{proof}
We first prove the following results.
\begin{lemma}
With sufficiently large $L$, $\forall j$ we have, 
\begin{align*}
P(\tilde{a}_j \neq \tilde{b}_j) \leq \frac{2}{t^2}~, ~\text{if}~~b_j = a_j~.
\end{align*}
\end{lemma}

\begin{proof}
Based on the definition of $a_j, b_j$ when $a_j = b_j$ we have the following hold.  
\begin{align*}
a_j &= E[X_j] = b_j~,c_j \geq E[(X_j-a_j)^+]~, c_j \geq E[(b_j-X_j)^+]~.
\end{align*}
Suppose we have $|E[\tilde{X}_j] - E[X_j]| \leq \epsilon$ (as proved in previous lemma with sufficiently large $L$) we therefore have 
\begin{align*}
E[(\tilde{X}_j &-E[\tilde{X}_j])^+]< E[(\tilde{X}_j - E[X_j])^+] + \epsilon \nonumber \\
&< E[(X_j - E[X_j])^+] + 2\epsilon\leq c_j +\frac{1}{t^{z/2}}~,
\end{align*}
\begin{align*}
E[(E[\tilde{X}_j]&-\tilde{X}_j)^+] < E[(E[X_j]- \tilde{X}_j)^+] + \epsilon \nonumber \\
&< E[(E[X_j] - X_j)^+] + 2\epsilon\leq c_j +\frac{1}{t^{z/2}}~,
\end{align*}
as long as 
$
\epsilon < \frac{1}{2t^{z/2}}
$; from which we have
$
\tilde{b}_j = \tilde{a}_j
$
based on the definition of $\tilde{a},\tilde{b}$s. 
\end{proof}

Similar with above proof we have the following results : 
\begin{lemma}
For sufficiently large $L$, at time $t$ we have $\forall \epsilon > 0$,
\begin{align*}
P(&|\frac{E[\tilde{X}_j]-c_j}{P(\tilde{X}_j \geq \tilde{a}_j)} -\frac{E[X_j]-c_j}{P(X_j \geq a_j)} | > C (\epsilon + \mathcal L\cdot (c_{1,j}+1)^{\alpha}\epsilon^{\alpha})) \leq \frac{2}{t^2}, \forall j~,
\end{align*}
for certain constant $C$. 
\label{lemma:bd1}
\end{lemma}
%
\begin{proof}
Consider the term 
$
\frac{E[\tilde{X}]-c_j}{P(\tilde{X} \geq \tilde{a}_j)} 
$ and we want to bound the estimation error associated with above terms, i.e., the probability, 
\begin{align*}
P(|\frac{E[\tilde{X}_j]-c_j}{P(\tilde{X}_j \geq \tilde{a}_j)} -\frac{E[X_j]-c_j}{P(X_j \geq a_j)} | > \epsilon)~.
\end{align*}

We need the following fact.
\begin{align}
|\frac{1}{x+\delta}-\frac{1}{x}| \leq \frac{1}{x^2} \cdot \delta, \forall x, \delta > 0.
\end{align}

For $P(\tilde{X}_j \geq \tilde{a}_j)$ we have
\begin{align}
&|P(\tilde{X}_j \geq \tilde{a}_j) - P(X_j \geq a_j)|  \nonumber \\
&\leq |P(\tilde{X}_j \geq \tilde{a}_j) - P(X_j \geq \tilde{a}_j)| + \epsilon \nonumber \\
&\leq |P(\tilde{X}_j \geq \tilde{a}_j) - P(X_j \geq \tilde{a}_j)| + \epsilon + \mathcal L\cdot (c_{1,j}\epsilon)^{\alpha}~.
\end{align}
The second relation comes from bounds on $\tilde{a}_j$ and Lipschitz condition of $F(\cdot)$. Plug in $x = P(X \geq a_j)$ we have 
\begin{align}
|\frac{E[\tilde{X}]-c_j}{P(\tilde{X} \geq \tilde{a}_j)} -\frac{E[X]-c_j}{P(X \geq a_j)} | \leq C \cdot (\epsilon + M\cdot (c_{1,j}\epsilon)^{\alpha})~.
\end{align}
for certain constant $C$.
\end{proof}

Denote
$$
\epsilon_3 = \min_{j \neq k} |E[X_j] - E[X_k]|,~\epsilon_4 = \min_{j \neq k} |\frac{E[X_j]-c_j}{P(X \geq a_j)}-\frac{E[X_k]-c_k}{P(X \geq a_k)}|~.
$$
Therefore when 
$$
C \cdot (\epsilon + \mathcal L\cdot (c_{1,j}+1)^{\alpha}\epsilon^{\alpha}) \leq \frac{\min \{\epsilon_3,\epsilon_4\}}{2}~,
$$
there is no error with the ordering (as similarly argued in ordering $a_j$s). Denote a solution for above $\epsilon$ as $\epsilon_o$ (which is trivial to show to exist). Then we further require 
$
L\cdot t^z \geq \frac{1}{\epsilon^2_o}~,
$
to guarantee a $O(1/t^2)$ error. 
\end{proof}

\section*{ Proof for Lemma \ref{lemma:vf}}
~\\
\begin{proof}
We prove by induction and the induction is based on the size of $S$. When $|S|=1$ (as well as $|\tilde{S}|$; also do notice $S = \tilde{S}$ due to the sorting algorithm we adopted. However due to the calculation of $a_j, b_j$ and inaccurate measure of $X_j$s, there is still discrepancy between the two value functions), we have (suppose we have $S=\{j\}$),
\begin{align*}
\tilde{V}(x,\tilde{S}) = \max\{-c_j+E[\tilde{V}(\max\{x, \tilde{X}_j\}&,\emptyset)], x,\max_{j \in \tilde{S}} E[\tilde{X}_j]\}~.
\end{align*}
Notice that 
$
\tilde{V}(\max\{x, \tilde{X}_j\},\emptyset) = \max\{x, \tilde{X}_j\}~,
$
we then have,
\begin{align*}
\tilde{V}(x,\tilde{S}) = \max\{-c_j+E[\max\{x, \tilde{X}_j\}], x,\max_{j \in \tilde{S}} E[\tilde{X}_j]\}~.
\end{align*}
Since with probability at least $1-\frac{2}{t^2}$ we have,
\begin{align*}
|E[\max\{x, \tilde{X}_j\}]-E[\max\{x, X_j\}]| &\leq \epsilon,~|E[\tilde{X}_j] - E[X_j]| \leq \epsilon~,
\end{align*}
we know w.h.p.
\begin{align}
|\tilde{V}(x,\tilde{S}) -V(x,S) | \leq \epsilon~,
\end{align}
since each term in the $\max$ function is bounded within the $\epsilon$-confidence region.
We therefore established the induction basis. Now suppose this is true for $|S| = k, k < N$. Consider the case with $|S| = k+1$. Based on the dynamic programming equations we know,
\begin{align*}
\tilde{V}(x,\tilde{S}) &= \max \{\max_{j \in \tilde{S}}\{-c_j+E[V(\max(x,\tilde{X}_j),\tilde{S}-j)]\}, x, \max_{j \in \tilde{S}} E[\tilde{X}_j]\}~.
\end{align*}
By induction hypothesis we know with probability at least $1-\frac{2}{t^2}$,
\begin{align*}
|\tilde{V}(\max(x,\tilde{X}_j),&\tilde{S}-j)-V(\max(x,\tilde{X}_j),S-j)| \leq k\cdot \epsilon~,
\end{align*}
and the fact
$
|E[\tilde{X}_j]-E[X_j]| \leq \epsilon~.
$
Therefore 
\begin{align}
\tilde{V}(x,\tilde{S}) \leq \max \{\max_{j \in \tilde{S}}\{-c_j&+E[V(\max(x,\tilde{X}_j),S-j)]\}\nonumber \\
&, x, \max_{j \in \tilde{S}}E[\tilde{X}_j]\} + k\cdot \epsilon
\end{align}
Also it is easy to notice with $x$ and $S-j$ being fixed, $V(\max(x,\tilde{X}_j),S-j)$ is also IID w.r.t. $X_j$. Then we have (via Chernoff-Hoeffding bound) 
$
|V(\max(x,\tilde{X}_j),S-j) - V(\max(x,X_j),S-j)| \leq \epsilon~.
$
Therefore 
\begin{align*}
\tilde{V}(x,&\tilde{S}) \leq \max \{\max_{j \in \tilde{S}}\{-c_j+E[V(\max(x,X_j),S-j)]\}\nonumber \\
&, x, \max_{j \in \tilde{S}}\} + (k+1)\cdot \epsilon=V(x, S) +  (k+1)\cdot \epsilon~.
\end{align*}
The other side of the inequality could be similarly proved and we finished the proof. 
\end{proof}

\section*{Proof for Lemma \ref{theorem:ds}}\label{appendix-E}
~\\
\begin{proof}
To prove this first notice the following dynamic equation holds for solving $d_s$,
\begin{align*}
V(0,S) = -c_1 + E[V(\max\{d_s, X_1\}, S-\{1\})]~.
\end{align*}
Consider the LHS by the above results we have the probability of
$
|\tilde{V}(0,\tilde{S}) - V(0,S) | > |S| \cdot \epsilon
$
being bounded by $\frac{2}{t^2}$. Consider then the case with
$
|\tilde{V}(0,\tilde{S}) - V(0,S) | \leq |S| \cdot \epsilon~.
$

For the RHS first notice
\begin{align*}
&|E[\tilde{V}(\max\{\tilde{d}_s, \tilde{X}_1\}, \tilde{S}-\{1\})]- E[V(\max\{\tilde{d}_s, \tilde{X}_1\}, S-\{1\})]| < |S| \cdot \epsilon~.
\end{align*}
We next show there exits a $\delta = O(\epsilon)$ such that
$
|\tilde{d}_s - d_s| < \delta~.
$
We prove this by contradiction.  Suppose $\tilde{d}_s \geq d_s + \delta$. We first would like to show the following
\begin{align}
&E[\tilde{V}(\max\{\tilde{d}_s, \tilde{X}_1\}, \tilde{S}-\{1\})] \nonumber \\
&- E[V(\max\{d_s, X_1\}, S-\{1\})] > -|S| \cdot \epsilon
\end{align}
To see this first notice 
\begin{align*}
&E[\tilde{V}(\max\{\tilde{d}_s, \tilde{X}_1\}, \tilde{S}-\{1\})] - E[V(\max\{d_s, X_1\}, S-\{1\})] \nonumber \\
&=(E[\tilde{V}(\max\{\tilde{d}_s, \tilde{X}_1\}, \tilde{S}-\{1\})]-E[\tilde{V}(\max\{d_s, \tilde{X}_1\}, \tilde{S}-\{1\})]) \nonumber \\
&+ (E[\tilde{V}(\max\{d_s, \tilde{X}_1\}, \tilde{S}-\{1\})]- E[V(\max\{d_s, X_1\}, S-\{1\})] )~.
\end{align*}
Since 
\begin{align}
&E[\tilde{V}(\max\{d_s, \tilde{X}_1\}, \tilde{S}-\{1\})]\nonumber \\
& - E[V(\max\{d_s, X_1\}, S-\{1\})] > -(|S|+1) \cdot \epsilon~,
\end{align}
it is sufficient to prove that 
\begin{align}
&E[\tilde{V}(\max\{\tilde{d}_s, \tilde{X}_1\}, \tilde{S}-\{1\})]\nonumber \\
&-E[\tilde{V}(\max\{d_s, \tilde{X}_1\}, \tilde{S}-\{1\})] > (2 |S|+1) \cdot \epsilon~.
\end{align}
Notice 
\begin{align*}
&E[\tilde{V}(\max\{\tilde{d}_s, \tilde{X}_1\}, \tilde{S}-\{1\})]-E[\tilde{V}(\max\{d_s, \tilde{X}_1\}, \tilde{S}-\{1\})] \nonumber \\
&\geq E[V(\max\{\tilde{d}_s, X_1\}, \tilde{S}-\{1\})]-E[V(\max\{d_s, X_1\}, \tilde{S}-\{1\})]-2\epsilon \nonumber \\
&\geq C_{d_s}\cdot \delta -2\epsilon~,
\end{align*}
where $C_{d_s}$ is a positive constant.  Therefore select $\delta$ large enough such that $C_{d_s}\cdot \delta > (2 |S|+3) \cdot \epsilon$ we finish the proof. Similarly we can prove the case for $\tilde{d}_s \leq d_s - \delta$.  We finish the proof. 
\end{proof}

\section*{Proof for Online.STEP 3.3.3}
\paragraph{On the sign of $E[\tilde{X}_1] - \tilde{g}_2(0)$}

Since the two cases with the sign are symmetric we will only prove the case when $E[X_1] - g_2(0) > 0$. Since
\begin{align*}
&\tilde{g}_2(0) = -c_2 + E[V(\max(\tilde{X}_1,0),1)] \nonumber \\
&=-c_2 + E[V(\tilde{X}_1,1)] \leq -c_2 + E[V(X_1,1)] + \epsilon ~,
\end{align*}
and $E[\tilde{X}_1] \geq E[X_1] - \epsilon$. Therefore as long as
$
\epsilon > \frac{E[X_1] - g_2(0)}{2}
$
we proved the claim. 

\paragraph{On $\tilde{b}_0$}

The proof is similar with the one for $d_s$ : bounding the estimation error for equations leading to the solution of $b_0$. Since $b_0$ satisfy the following equality:
\begin{align}
g_1(b_0) = \max\{E[X_1],g_2(0)\}~.
\end{align}
Consider LHS $g_1(b_0)$. If $\tilde{b_0} \geq b_0 + \delta$ we have 
\begin{align}
g_1(\tilde{b}_0) \geq g_1(b_0) + C_{b_0} \cdot \delta - \epsilon~.
\end{align}
for certain constant $C_{b_0}$. Consider RHS we have 
\begin{align}
|\max\{E[\tilde{X}_1],\tilde{g}_2(0)\}-\max\{E[X_1],g_2(0)\}| \leq \epsilon~.
\end{align}
Therefore if $ \delta > \frac{2\epsilon}{C_{b_0}}$ we arrive at contradiction. 


\end{document}